\documentclass[runningheads]{llncs}
\usepackage[T1]{fontenc}
\usepackage{graphicx}
\usepackage{commands} 
\usepackage{amssymb} 
\usepackage{amsmath} 
\usepackage{subfig}
\usepackage{subfloat}
\begin{document}

\title{Responsibility in a Multi-Value Strategic Setting}
%
%
\author{Timothy Parker$^*$\orcidID{0000-0002-5594-9569} \and
Umberto Grandi\orcidID{0000-0002-1908-5142} \and
Emiliano Lorini\orcidID{0000-0002-7014-6756}}
\authorrunning{T. Parker et al.}
%
\institute{IRIT, CNRS, University of Toulouse, Toulouse, France
\email{\{timothy.parker,umberto.grandi,emiliano.lorini\}@irit.fr}\\
$^*$: Corresponding Author}
\maketitle              
\begin{abstract}
Responsibility is a key notion in multi-agent systems and in creating safe, reliable and ethical AI.  In particular, the evaluation of choices based on responsibility is useful for making robustly good decisions in unpredictable domains. However, most previous work on responsibility has only considered responsibility for single outcomes, limiting its application. In this paper we present a model for responsibility attribution in a multi-agent, multi-value setting. We also expand our model to cover responsibility anticipation, demonstrating how considerations of responsibility can help an agent to select strategies that are in line with its values. In particular we show that non-dominated regret-minimising strategies reliably minimise an agent's expected degree of responsibility.
\end{abstract}

\section{Introduction}

Responsibility attribution \cite{AlechinaHL17,Baier0M21,DBLP:journals/jair/ChocklerH04,Halpern15} is the process of determining which agent or set of agents can be held responsible for a particular outcome. This is a backward-looking process, meaning that while it is useful for allocating praise or blame, it cannot be used for strategy selection, since responsibility for an outcome can only be determined once the outcome has occurred. Responsibility anticipation \cite{ECAIpaper} is the process of predicting which outcomes an agent \emph{may} be responsible for if it performs a particular strategy. This means it can be used in strategy selection, i,e to ensure that an agent cannot be responsible for some negative outcome. In particular, considerations of responsibility are particularly useful in settings where the satisfaction of some crucial value cannot always be guaranteed, since by minimising its anticipated responsibility for the violation of a value, an agent can ensure that it maximises its causal contribution towards the satisfaction of the value (since if the agent could have done more to try and satisfy the value, then it would anticipate responsibility for the value's violation).

However, current methods for responsibility anticipation only consider responsibility for single outcomes or values. In a real-world setting, it is likely that an agent will have to consider multiple, possibly conflicting values. Therefore we wish to investigate how notions of responsibility can be expanded to work in a multi-value setting, to widen the scope for the application of responsibility-based evaluation in various settings.

In this paper we present and discuss various properties that are appealing for multi-value notions of responsibility. Of particular interest is the idea that an agent may be able to avoid responsibility in some cases by providing an ``excuse'' justifying their choice of a specific strategy. We will introduce two separate definitions of responsibility, one simple notion of passive responsibility and a more complex notion (inexcusable passive responsibility) that considers the excuses that the agent might have. The main result of our paper is that these notions of responsibility are all equivalent in anticipation to some combination of non-dominated and regret-minimising strategies (for which we introduce a symbolic notion of regret in this paper). In particular, strategies that minimise both forms responsibility are exactly the strategies that are both non-dominated and regret-minimising. 

The rest of the paper is organised as follows, section \ref{sec:relatedwork} locates our paper in its field and compares our work with some related approaches from the literature. Section \ref{sec:model} introduces our formal model and section \ref{sec:responsibilityatt} defines our notions of responsibility. Section \ref{sec:regret} formalises our notion of symbolic regret, and considers the perspective of anticipation. Finally, section \ref{sec:futurework} concludes the paper and outlines directions for future work.

\begin{example}[Cleaning the Shopping Centre]
    A robot (Anna) is tasked with various jobs in cleaning a shopping centre. There is a customer (Ben) who can help or hinder Anna's tasks depending on how he acts, such as by blocking access to a bin which Anna is trying to empty (for simplicity, we consider only a single customer). The jobs that Anna must complete are as follows:
    \begin{itemize}
        \item $\omega_1 =$ The plants are watered.
        \item $\omega_2 =$ The bins are emptied.
        \item $\omega_3 =$ The windows have been cleaned.
        \item $\omega_4 =$ Litter has been collected.
        \item $\omega_5 =$ The floor has been swept.
    \end{itemize}
\end{example}

\section{Related Work}\label{sec:relatedwork}

This paper covers responsibility attribution and anticipation, and how this can be expanded to handle responsibility in a multi-value system. 

Responsibility attribution is a very well studied topic with various approaches taken by a wide variety of authors, including approaches based on game-theoretic tools \cite{Baier0M21,Braham2012,DBLP:journals/fuin/LoriniM18} and logical tools including STIT logic \cite{DBLP:journals/logcom/LoriniLM14,DBLP:conf/atal/AbarcaB22,Baltag2021-BALCAA-8,DBLP:journals/ai/LoriniS11}, LTLf \cite{ECAIpaper},  ATL \cite{DBLP:conf/atal/YazdanpanahDJAL19,DBLP:conf/clima/BullingD13}, logics of strategic and extensive games \cite{DBLP:conf/aaai/Shi24,DBLP:conf/ijcai/Naumov021,DBLP:journals/apal/NaumovT23} and
structural equation models  \cite{DBLP:journals/jair/ChocklerH04}. Responsibility anticipation is less well studied, though there is some previous work by Grandi et al. \cite{IJCAIpaper,ECAIpaper}. Overall, our model is most similar to the one presented by Grandi et al. as our objective was to expand their ideas to account for responsibility in a multi-value setting.

A similar notion to responsibility anticipation is responsibility for risk \cite{vandePoel2012}, as both concern an agent's responsibility for an outcome that may not actually happen. The main differences are that risk is typically formalised in an explicitly probabilistic way, such as in the work of Gladyshev et al. \cite{KR2023-32} who introduce a logical model of responsibility for risk that primarily considers ex post attribution of responsibility for risk (whereas our notion is ex ante) and introduce a complete and decidable logic for reasoning about group responsibility for taking risk.
 
 To the best of our knowledge, there is very little work that investigates responsibility in such a setting. One work that we are aware of is an approach by Lorini and Sartor \cite{DBLP:conf/jurix/LoriniS21} whose work focuses on secondary responsibility, which is the responsibility of an agent who influences another agent to perform some action. Their notions do not share the properties of consistency or completeness that we discuss, since they consider responsibility \textit{relative to} a set of values (meaning the agent is not implied to be responsible for everything in the set) rather than our work that focuses on responsibility \textit{for} a set of values. They do also consider the principle that an agent is more responsible if they have no excuse for their actions, though their notion of excuse is based on the values that an agent controls rather than domination between strategies.

We also introduce a symbolic notion of regret, for use in strategy comparison. This was first introduced (independently) in decision theory by Savage \cite{Savage51} and Niehans \cite{Niehans1948}. It was later introduced to game theory by Linhart and Radner \cite{Linhart89}. Our concept of regret is also similar to the notion of guilt introduced by Lorini and M{\"u}hlenbernd \cite{DBLP:journals/fuin/LoriniM18}. Their model uses separate numerical values to track both the individual utility of an agent for some particular history, as well as the degree of ideality of that history (such as the utility of the worst-off agent). The guilt of an agent in a history is the difference between the ideality achieved and the best possible ideality (fixing the actions of all other agents). This is similar to our notion of regret based on passive responsibility, but is purely numeric instead of symbolic.

\section{Model}\label{sec:model}

In this section we introduce the framework for our model. We require a finite set of agents $\agentset = \{1,\ldots,n\}$ and a countable set of propositions $\propset = \{p, q, \ldots\}$ which produces a set of states $\stateset = 2^{\propset}$. Let $\actset = \{a, b , \ldots\}$ be a finite non-empty set of action names. To describe the actions taken by all agents at a single time we introduce the notion of a joint action, which is a function $J:\agentset \longrightarrow \actset$. The set of all joint actions is $\jointactset$.

To trace the actions of agents and changing states over time we define a $k$-history to be a pair $\history= (\historyst, \historyact)$ with $\historyst: \{0,\ldots,k\} \longrightarrow \stateset$ and $\historyact: \{0,\ldots,k-1\} \longrightarrow \jointactset$. The set of $k$-histories is noted $\historyset{k}$. The set of all histories is $\historyset{}=\bigcup_{k \in \nat }\historyset{k}$. For convenience, given a $k$-history $\history$ and some $k' \leq k$ we write $\history^{k'}$ for the history corresponding to the first $k'+1$ states of $\history$. 

To describe the world in which our agents operate we introduce the notion of a multiagent transition system (MTS).
\begin{definition}[Multiagent Transition System]
    A multiagent transition system $\game$ is a pair $(\stateset,\actionfunction)$ where $\actionfunction: \stateset \times \jointactset \to \stateset$ is a function that maps each pair of a state $s$ and a joint action $J$ to a successor state of $s'$.
\end{definition}


In our model, agents act according to strategies that determine how they should act based on the actions of all agents and how the state of the world has progressed. More formally, given an MTS $\game = (\stateset,\actionfunction)$ a strategy is a function $\strategy: \historyset{} \to \actset$. The set of all strategies is denoted $\stratset$. A joint strategy for a coalition $J \subseteq \agentset$ where $J = \{j_1,\ldots,j_m\}$ is a tuple $\jstrategy = (\sigma_{j_1},\ldots,\sigma_{j_m})$. The set of all joint strategies for the coalition $J$ is written $\jstratset^J$.Given joint strategies $\jstrategy$ and $\jstrategy'$ for coalitions $J_1$ and $J_2$ where $J_1 \cap J_2 = \emptyset$, we write $(\strategy_i,\strategy_i')$ for the union of $\strategy_i$ and $\strategy_i'$. For convenience, we write $\jstrategy_{-i}$ for the reduction of $\jstrategy$ from the coalition $J$ to the coalition $J \setminus \{i\}$ (when $i \in J$).

In our model, histories are temporal entities that are always finite in length, therefore the most natural choice to describe properties of histories is Linear Temporal Logic over Finite Traces \cite{GiacomoV13,degiacomo2015}. This allows us to describe temporal properties such as ``$\phi$ never occurs'' or ``$\phi$ always occurs immediately after $\psi$''. We write the language as $\langlogic_{\ltllogic}$, defined by the following grammar:
\begin{center}\begin{tabular}{lcl}
  $\phi $ &  $\bnf$ & $ p  \mid  \neg\phi \mid \phi  \wedge \phi   \mid \nexttime \phi \mid
\until { \phi   } { \phi    },  $\\
\end{tabular}\end{center}
with 
$p$ ranging  over $\propset$. Atomic formulas in this language are those that consist of a single proposition $p$.
$\nexttime$
and $\until {     } {     }$
are the  operators
``next''
and ``until'' of $\ltllogic$. 
Operators
``henceforth'' ($\henceforth$)
and ``eventually'' ($\eventually$)
are defined in the usual way:
$\henceforth \phi \defin \neg ( \until {   \top   } {  \phi     }) $
and 
$\eventually \phi  \defin \neg  \henceforth \neg  \phi $. We define the semantics for $\nexttime$ and $\until{}{}$ as follows, the rest is the same as $\langlogic_{\proplogic+}$ (for $t \in \{0,\ldots,k\}$).
\begin{alignat*}{2}
  \history,  t &\models \nexttime \phi & ~\IFF~ & t < k \AND \history,  t+1 \models   \phi, \\
   \history,  t &\models \until { \phi_1    } { \phi_2    }   & ~\IFF~ &
    \begin{aligned}[t]
      &\exists t' \geq t :   t' \leq k \AND \history,  t' \models \phi_2 \AND\\
      &\forall t''   \geq t  :   \IF 
      t'' < t' ~\THEN~ \history,  t'' \models \phi_1.
    \end{aligned} 
\end{alignat*}

\subsection{Values, Goals and Moral Action Systems}

We assume that our agents will have multiple goals and/or values that they wish to satisfy, which may have different priority levels. Following the approach of Grandi et al. \cite{GrandiLPA22}. We represent this by a prioritised value base $\valueprof$ which is simply a sequence of sets of $\ltllogic$-formulas $\valueprof = (\valueset_1,...,\valueset_m)$. We do not consider the process by which these values are arrived at, nor how they are translated into $\ltllogic$, as that is outside the scope of this paper. The content of the values that artificial agents should follow is a relatively well-studied area \cite{Talbot17,BelloniG15,Rossimoralpref}.

For each $\valueset_n \in \valueprof$ we write $\valuesetminus_n$ for the set $\{\neg \omega \suchthat \omega \in \valueset_n\}$. Finally we write $\valueprof^{\neg}$ for the tuple $(\valuesetminus_1, \ldots, \valuesetminus_m)$ and $\valueprofplus$ for the set $\bigcup \valueprof \cup \bigcup \valueprof^{\neg}$. We assume that our value base is ``consistent'' meaning that for all $\omega_1, \omega_2 \in \bigcup \valueprof$, $\omega_2 \not \equiv \neg \omega_1$. As we are concerned with an agent's responsibility for both the satisfaction and violation of values we will be mostly considering comparing subsets of $\valueprofplus$.

To represent the structure of $\valueprof$ we use quantitative lexicographic comparison over the sets in $\valueprof$, treating the first set $\valueset_1$ as most important and the last set $\valueset_m$ as the least. This comparison works as follows, given $X, Y \subseteq \valueprofplus$, we say that $X \preceq Y$ if and only if one of the two following properties holds:
\begin{align*}
  i) & \ \exists n  \text{ s.t } 1 \leq n \leq m
   \text{ and }  (|X \cap \valueset_n| - |X \cap \valuesetminus_n|) < \\&(|Y \cap \valueset_n| - |Y \cap \valuesetminus_n|)\\ &\text{and}\\
          & \ \forall n' \text{ if }  1 \leq n' < n
          \text{ then }
           (|X \cap \valueset_n| - |X \cap \valuesetminus_n|) = \\&(|Y \cap \valueset_n| - |Y \cap \valuesetminus_n|)\\
      ii)   &  \ \forall n
      \text{ if }
      1 \leq n \leq m      \text{ then } (|X \cap \valueset_n| - |X \cap \valuesetminus_n|) = \\&(|Y \cap \valueset_n| - |Y \cap \valuesetminus_n|)
\end{align*}

We also define $X \valuelt Y$ as $X \valueleq Y$ and $Y \not \valueleq X$. To allow for comparing histories we write $\satset(\history,\valueprofplus)$ for the maximal subset of $\valueprofplus$ that is satisfied in $\history$.

There are many alternative ways to rank sets of values for for strategy (or plan) comparison \cite{Bienvenu06,Dennis2} which could also be used in our model.

\begin{definition}[Moral Action System]
A Moral Action System (MAS) is a tuple $\pldomain=(\game,s_0,k,\valueprof)$ where $\game$ is an MTS, $s_0 \in \stateset$ is a start state, $k$ is an integer (known as the ``horizon'') and $\valueprof$ is a value base.
\end{definition}

For simplicity, in any MAS $\pldomain$ it is assumed that $\game$, $s_0$ and $k$ are fully observable. Also for simplicity, the value base $\valueprof$ is considered to be shared by all agents, though the agents have no information about the values of other agents and cannot communicate. For this reason we consider only the perspective of a single agent in this paper. Furthermore, it is entirely consistent to suppose that $\valueprof$ contains only a single value set that contains only a single value, meaning that all of the results in this paper can, in principle, be applied to single-value settings. However, in practice it is likely to be more straightforward to use single-value notions of responsibility \cite{Braham2012,DBLP:journals/logcom/LoriniLM14,ECAIpaper} in single-value settings.

If $\jstrategy$ is a strategy for $\agentset$ we write $\play(\jstrategy,\pldomain)$ for the history resulting from the application of $\jstrategy$ to the moral action system $\pldomain = (\game,s_0,k,\valueprof)$. This is the history that starts at $s_0$ and proceeds according to $\jstrategy$ and $\actionfunction$ for $k$ steps.

More formally, given a strategy $\jstrategy$ for $\agentset$ and an MAS $\pldomain = (\game,s_0,k,\valueprof)$, $\history = \play(\jstrategy,\pldomain)$ is the $k$-history such that:
\begin{align*}
    \historyst(0) = &s_0\\
    \forall 1 \leq k' \leq k, \historyact(k)(i) = &\jstrategy_i(\history^{k'-1})\\
    \forall 1 \leq k' \leq k, \historyst(k) = &\actionfunction(\historyst(k'-1),\historyact(k'))
\end{align*}

\begin{example}[Continued]
    Anna's value base treats all tasks as equally important, meaning that $\valueprof = (\valueset_1)$ where $\valueset_1 = \{\omega_1, \omega_2, \omega_3, \omega_4, \omega_5\}$
\end{example}

The main aim of this paper is to demonstrate that considering notions of responsibility is helpful in the selection of strategies in a multi-value setting. One notion that it is therefore helpful to define is the notion of a non-dominated strategy, which is a standard concept in game theory and decision theory. Firstly, we define weak dominance:

\begin{definition}[Weak Dominance]
   Given an MAS $\pldomain = (\game,s_0,k,\valueprof)$, an agent $i$ and two strategies $\strategy_i$ and $\strategy_i'$, we say that $\strategy_i'$ weakly dominates $\strategy_i$ (written $\strategy_i \leq_\pldomain \strategy_i'$) if and only if for all $\jstrategy_{-i} \in \jstratset^{-i}$, $\satset(\play((\strategy_i,\jstrategy_{-i}),\pldomain),\valueprofplus) \preceq \satset(\play((\strategy_i',\jstrategy_{-i}),\pldomain),\valueprofplus)$. 
\end{definition}

In words, $\strategy_i'$ weakly dominates $\strategy_i$ if and only if in every possible outcome from executing $\strategy_i'$, (i.e for every possible strategy for $\agentset \setminus \{i\}$, the outcome obtained is at least as good as the outcome that would have been obtained by executing $\strategy_i$ instead.

\begin{definition}[Non-Dominated]
     Given an MAS $\pldomain = (\game,s_0,k,\valueprof)$, we say that $\strategy_i$ is non-dominated if and only if there is no strategy $\strategy_i' \in \stratset$ such that $\strategy_i \leq_\pldomain \strategy_i'$ and $\strategy_i' \not \leq_\pldomain \strategy_i$.
\end{definition}

In words, a strategy $\strategy_i$ is non-dominated if and only if there is no strategy $\strategy_i'$ such that $\strategy_i'$ weakly dominates $\strategy_i$ and sometimes does strictly better. In particular, since there are finitely many possible histories in any MAS $\pldomain = (\game,s_0,k,\valueprof)$ (since $k$ must be finite), there are finitely many distinct strategies, so there must always exist at least one non-dominated strategy.

\section{Responsibility Attribution}\label{sec:responsibilityatt}

In this paper we will primarily be focusing on the notion of passive responsibility, as formalised by Lorini et al. \cite{DBLP:journals/logcom/LoriniLM14} and Grandi et al. \cite{IJCAIpaper,ECAIpaper}. This corresponds to the notion of ``allowing $\omega$ to happen'' and is defined as follows: Given an MAS $\pldomain = (\game,s_0,k,\valueprof)$, a joint strategy $\jstrategy$ for $\agentset$ and an agent $i$ we say that $i$ is attributed single-value passive responsibility for $\omega$ in $\play(\jstrategy,\pldomain)$ under $\jstrategy$ if and only if $\play(\jstrategy,\pldomain) \models \omega$ and there is some strategy $\strategy'_i$ for $i$ such that $\play((\strategy'_i,\jstrategy_{-i}),\pldomain) \models \neg \omega$. In words, an agent is attributed passive responsibility for $\omega$ if $\omega$ occurs and, fixing the strategies of all other agents, $i$ could have acted differently and prevented $\omega$.

Another way to think about responsibility is from a courtroom-esque perspective where some external agent (the accuser) is trying to blame $i$ for the occurrence of $\omega$. For this we define the notion of \textit{accusation}, where given a joint strategy $\jstrategy$ for $\agentset$, an agent $i$, an $\LTLf$ formula $\omega$ and an MAS $\pldomain = (\game, s_0, k, \valueprof)$ such that $\play(\jstrategy,\pldomain) \models \omega$, a single-value accusation for $(i,\jstrategy,\omega,\pldomain)$ is a strategy $\strategy_i'$ such that $\play((\strategy_i',\jstrategy_{-i}),\pldomain) \models \neg \omega$. We can then see that given a joint strategy $\jstrategy$ for $\agentset$, $i$ is attributed single-value passive responsibility for $\omega$ in $\play(\jstrategy,\pldomain)$ if and only if there exists an accusation for $(i,\jstrategy,\omega,\pldomain)$.

In words, we can imagine that the accuser says to the agent ``you are responsible for $\omega$ because you could have prevented it, if you had acted according to $\strategy_i'$ then $\omega$ would not have occured''. However, we might think that for this accusation to be truly effective, there can be no reason to pick $\strategy_i$ over $\strategy_i'$, which we take to mean that $\strategy_i'$ must weakly dominate $\strategy_i$.

\begin{definition}[Liability]\label{def:liability}
    Given a joint strategy $\jstrategy$ for $\agentset$, an agent $i$, an $\ltllogic$ formula $\omega$ and an MAS $\pldomain = (\game,s_0,k,\valueprof)$ we say that $i$ is liable for $\omega$ in $\play(\jstrategy,\pldomain)$ if and only if $\play(\jstrategy,\pldomain) \models \omega$ and there exists some strategy $\strategy_i'$ such that $\play((\strategy_i',\jstrategy_{-i}),\pldomain) \models \neg \omega$ and $\strategy_i \leq_{\pldomain} \strategy_i'$. 
\end{definition}

One possible issue with the notion of single-value responsibility given above is that attributing responsibility requires full knowledge of the strategies of all agents, which is a very strong demand, particularly if some of those agents are humans. We do not address this issue in our paper since we focus more on anticipating than attributing responsibility. When anticipating possible outcomes an agent will quantify over all possible strategies of the other agents, and assuming knowledge of all agents' strategies in a simulated execution is not problematic. Nonetheless, considering how we can attribute responsibility with limited information about the strategies of agents would be an interesting direction for future work.

\subsection{Multi-Value Responsibility}

We shall now consider how to evaluate the responsibility of an agent in a setting with multiple values, with a particular focus on how we can use these notions of responsibility to help agents select good strategies.

Perhaps the simplest approach to multi-value responsibility is to say that $i$ is responsible for some set of values $\Omega$ if and only if $i$ is responsible for each $\omega \in \Omega$. This seems reasonable, as it means that if $i$ is responsible for $\Omega$ then $i$ could have made every formula in $\Omega$ true if it had acted differently. However, this approach does not consider if $i$ could have made the formulas in $\Omega$ true \textit{simultaneously}, only \textit{individually}. This can cause unintuitive results.

\begin{table}[t]
\centering
\subfloat[Table 1]{
\begin{tabular}{|c|c|}
    \hline
    & $\strategy_{B}$\\
    \hline
    $\strategy_A$ & $\emptyset$\\
    \hline
    $\strategy_A'$ & $\omega_1$\\
    \hline
    $\strategy_A''$ & $\omega_2$\\
    \hline
    $\strategy_A'''$ & $\omega_1, \omega_2$\\
    \hline
\end{tabular}
}
\quad
\subfloat[Table 2]{
\begin{tabular}{|c|c|c|}
    \hline
    & $\strategy_{B}$ & $\strategy_{B}'$\\
    \hline
    $\strategy_A$ & $\omega_1$, $\omega_2$ & $\omega_3$\\
    \hline
    $\strategy_A'$ & $\emptyset$ & $\omega_1$, $\omega_2$ \\
    \hline
\end{tabular}}
\end{table}

Consider Table 1, which presents a very simple scenario where Anna has to pick between various strategies but Ben has only one strategy. We can consider three scenarios based on this table. In scenario A, Anna only has access to $\strategy_A$ and $\strategy_A'$. In scenario B Anna also has access to $\strategy_A''$ and in scenario C Anna has access to all four strategies. It is intuitive to assert that in scenario A, if Anna picks $\strategy_A$ then she should be responsible only for $\{\neg \omega_1\}$ whereas in scenario C she should be responsible for $\{\neg \omega_1, \neg \omega_2\}$ (if she chose $\strategy_A$). This is because in scenario A the best Anna could have done was to water the plants $\{\omega_1\}$ whereas in scenario B she could have watered the plants and emptied the bins $\{\omega_1,\omega_2\}$. However, it is less clear what should happen in scenario B. Separately, it is clear that we can hold Anna responsible for not watering the plants \textit{or} for not emptying the bins, but if we hold her responsible for both, then we find that she is equally responsible in scenario B as scenario C. This is odd because it seems that the addition of $\strategy_A'''$ increases Anna's degree of responsibility much more than the addition of $\strategy_A''$. Accordingly, we introduce the following requirement on multi-value notions of responsibility which takes the form of a necessary (but not necessarily sufficient) requirement for responsibility:

\begin{property}[Consistency] Given an MAS $\pldomain = (\game,s_0,k,\valueprof)$, an agent $i$ and a strategy $\jstrategy$ for $\agentset$, if agent $i$ is responsible for $X$ in the history $\play(\jstrategy,\pldomain)$ then $\play(\jstrategy,\pldomain) \models \bigwedge X$ and there exists some $\strategy_i'$ such that $\play((\strategy_i',\jstrategy_{-i}),\pldomain) \models \bigwedge\{\neg \omega \suchthat \omega \in X\}$.
\end{property}

This means that in general, we cannot assume that if $i$ is responsible for the set $A$ and the set $B$ in the history $\play(\jstrategy,\pldomain)$ then $i$ is also responsible for the set $A \cup B$. Another consideration is the issue that a responsibility set $A$ may be ``misleading'' in the sense that it misses out important information. In Table 2 $i$ can seemingly be held responsible for $X = \{\omega_1,\omega_2\}$ in both the history $\history_1 = \play((\strategy_i,\jstrategy_{-i}'),\pldomain)$ and the history $\history_2 = \play((\strategy_i',\jstrategy_{-i}),\pldomain)$. However, it seems that $i$ should be ``more'' responsible in $\history_2$ than $\history_1$ since in $\history_1$ we at least have the consolation of satisfying $\omega_3$ which could not have been done otherwise, whereas in $\history_1$ there is no compensation. This suggests that responsibility sets should be not just ``consistent'' but also ``complete''.

\begin{property}[Completeness] Given an MAS $\pldomain = (\game,s_0,k,\valueprof)$, an agent $i$ and a strategy $\jstrategy$ for $\agentset$, if agent $i$ is responsible for $X \subseteq \langlogic_{\ltllogic}$ in the history $\history_1 = \play(\jstrategy,\pldomain)$ then there exists some strategy $\strategy_i'$ for $i$ such that for $\history_2 = \play((\strategy_i',\jstrategy_{-i}),\pldomain)$ and all $\omega \in \valueprof$ such that $\omega, \neg \omega \notin X$, either $\history_1, \history_2 \models \omega$ or $\history_1, \history_2 \models \neg \omega$.
\end{property}

To guarantee the notions of consistency and completeness we define the notion of responsibility via a strategy, which will be a useful building block in our future definitions.

\begin{definition}[Responsibility via Strategy]
    Given an MAS $\pldomain = (\game,s_0,k,\valueprof)$, an agent $i$ and a strategy $\jstrategy$ for $\agentset$, we say that $i$ is responsible for $X \subseteq \langlogic_{\ltllogic}$ via $\strategy_i'$ if and only if $X = \satset(\play(\jstrategy,\pldomain),\valueprofplus) \setminus \satset(\play((\strategy_i',\jstrategy_{-i}),\pldomain,\valueprofplus)$.
\end{definition}

We will now introduce a natural extension of passive responsibility that guarantees consistency and completeness:

\begin{definition}[Passive Responsibility]

Given an MAS $\pldomain = (\game,s_0,k,\valueprof)$, an agent $i$ and a strategy $\jstrategy$ for $\agentset$, $i$ is passively responsible for $X \subseteq \valueprofplus$ in $\play(\jstrategy,\pldomain)$ if and only if if there exists some strategy $\strategy_i'$ for $i$ such that $i$ is responsible for $X$ via $\strategy_i'$.
\end{definition}

\begin{theorem}\label{thm:passive_resp_att}
    Passive responsibility satisfies consistency and completeness.
\end{theorem}

\begin{proof}
    Let $\pldomain = (\game,s_0,k,\valueprof)$ be an MAS, $i$ an agent and $\jstrategy$ a strategy for $\agentset$. Suppose that $i$ is passively responsible for $X \subseteq \valueprofplus$ in $\history_1 = \play(\jstrategy,\valueprof)$. Therefore there exists some strategy $\strategy_i'$ such that $X = \satset(\history_1,\valueprofplus) \setminus \satset(\history_2,\valueprofplus)$ where $\history_2 = \play((\strategy_i',\jstrategy_{-i}),\pldomain)$.

    First we will show consistency. Since $X \subseteq \satset(\history_1,\valueprofplus)$ we know that $\history_1 \models \bigwedge X$. Furthermore, $X \cap \satset(\history_2,\valueprofplus) = \emptyset$ and $X \subseteq \valueprofplus$ so we know that for all $\omega \in X \cap \valueprof, \neg \omega \in \satset(\history_2,\valueprofplus)$ and for all $\neg \omega \in X \cap \valueprof^\neg, \omega \in \satset(\history_2,\valueprofplus)$. Therefore $\history_2 \models \bigwedge\{\neg \omega \suchthat \omega \in X\}$ and we are done.

    Second we will show completeness. Suppose that for some $\omega \in \valueprof$, $\omega \neg \omega \notin X$. Furthermore, we know by necessity that either $\omega \in \satset(\history_1,\valueprofplus)$ or $\neg \omega \in \satset(\history_1,\valueprofplus)$. If $\omega \in \satset(\history_1,\valueprofplus)$ then since $\omega \notin X$ we know that $\omega \in \satset(\history_2,\valueprofplus)$. By a similar argument for $\neg \omega$ we can show that either $\history_1, \history_2 \models \omega$ or $\history_1, \history_2 \models \neg \omega$ and we are done.
\end{proof}

\subsection{Excuses}
\begin{table}[t]
    \centering
\subfloat[Table 3]{\begin{tabular}{|c|c|c|}
    \hline
    & $\strategy_{B}$ & $\strategy_{B}'$\\
    \hline
    $\strategy_A$ & $\omega_1$, $\omega_2$ & $\emptyset$\\
    \hline
    $\strategy_A'$ & $\emptyset$ & $\omega_1$, $\omega_2$ \\
    \hline
\end{tabular}}
\quad
\subfloat[Table 4]{\begin{tabular}{|c|c|c|c|}
    \hline
    & $\strategy_{B}$ & $\strategy_{B}'$ & $\strategy_{B}''$\\
    \hline
    $\strategy_A$ & $\omega_1$, $\omega_2, \omega_3$ & $\omega_1, \omega_2$ & $\emptyset$\\
    \hline
    $\strategy_A'$ & $\emptyset$ & $\omega_1$, $\omega_2, \omega_3$ & $\omega_1, \omega_2$\\
    \hline
    $\strategy_A''$ & $\omega_1, \omega_2$ & $\emptyset$ & $\omega_1, \omega_2, \omega_3$\\
    \hline
\end{tabular}}
\end{table}

In Table 3 we see that choosing either $\strategy_A$ or $\strategy_A'$ may lead Anna to be responsible for $\{\omega_1,\omega_2\}$. This seems somewhat unfair, as the agents in my model have no way do determine which of $\strategy_{B}$ and $\strategy_B'$ was more likely to occur (since they have no information about the goals, values or rationality of the other agents), meaning that they always risk responsibility. Therefore, against the ``accusation'' of responsibility via some strategy $\strategy'_i$, an agent may excuse themselves by presenting a rational justification for choosing $\strategy_i$ over $\strategy'_i$. We consider excuses only in cases of ``negative responsibility'' (responsibility for $X$ where $X \preceq \emptyset)$ as it is not clear what it would mean to give an excuse for a positive outcome.

\begin{definition}[Weak Excuse]
    If agent $i$ is responsible for $X \preceq \emptyset$ via $\strategy_i'$ in $\play(\jstrategy,\pldomain)$, then a weak excuse for $(i,\jstrategy,\pldomain,\strategy_i')$ is a strategy $\jstrategy_{-i}'$ for $\agentset \setminus \{i\}$ such that $\satset(\play((\jstrategy_{i},\jstrategy_{-i}'),\pldomain,\valueprofplus)$ is strictly preferred to $\satset(\play((\strategy_i',\strategy_{-i}'),\pldomain,\valueprofplus)$.
\end{definition}

\begin{property}[Acceptance of Weak Excuses] Given an MAS $\pldomain = (\game,s_0,k,\valueprof)$, an agent $i$ and a strategy $\jstrategy$ for $\agentset$, if $i$ is responsible for $X \subseteq \langlogic_{\ltllogic}$ in $\play(\jstrategy,\pldomain)$ then there exists some strategy $\strategy_i'$ for $i$ such that $i$ is responsible for $X$ via $\strategy_i'$ and there exists no weak excuse for $(i,\jstrategy,\pldomain,\strategy_i')$.
\end{property}

A weak excuse says that $i$ was (weakly) justified in choosing $\strategy_i$ over $\strategy_i'$ since there is at least one strategy for $\agentset \setminus \{i\}$ where $\strategy_i$ does better. However, we might prefer a stronger notion of excuse, requiring that the original strategy could have been preferred \textit{by at least as wide a margin} as the accusing strategy.

\begin{definition}[Strong Excuse]
    If agent $i$ is responsible for $X \preceq \emptyset$ via $\strategy_i'$ in $\play(\jstrategy,\pldomain)$, then a strong excuse for $(i,\jstrategy,\pldomain,\strategy_i')$ is a strategy $\jstrategy_{-i}'$ for $\agentset \setminus \{i\}$ such that $\satset(\play((\strategy_i,\jstrategy_{-i}'),\pldomain),\valueprofplus) \setminus \satset(\play((\strategy_i',\jstrategy_{-i}'),\pldomain),\valueprofplus)$ is preferred to $\emptyset$ and to $\satset(\play((\strategy_i',\jstrategy_{-i}),\pldomain),\valueprofplus) \setminus \satset(\play(\strategy,\pldomain),\valueprofplus)$.
\end{definition}

Note that since being a strong excuse is a strictly stronger requirement than being a weak excuse, all strong excuses are automatically weak excuses.

\begin{property}[Acceptance of Strong Excuses]
   Given an MAS $\pldomain = (\game,s_0,k,\valueprof)$, an agent $i$ and a strategy $\jstrategy$ for $\agentset$, if $i$ is responsible for $X \subseteq \langlogic_{\ltllogic}$ in the history $\play(\jstrategy,\pldomain)$ then there exists some strategy $\strategy_i'$ for $i$ such that $i$ is responsible for $X$ via $\strategy_i'$ and there exists no strong excuse for $(i,\jstrategy,\pldomain,\strategy_i')$.
\end{property}

However, while we have outlined why we think that strong excuses might be a more appealing notion than weak excuses, there is also an issue with this notion. The property ``acceptance of strong excuses'' relies on the intuition that if in some history $\play(\jstrategy,\pldomain)$ an agent $i$ cannot give a strong excuse for choosing $\strategy_i$ over some alternative $\strategy_i'$ (where $\satset(\play(\jstrategy,\pldomain),\valueprofplus) \prec \satset(\play((\strategy_i',\jstrategy_{-i}),\pldomain),\valueprofplus)$) then they cannot justify their choice of $\strategy_i$ over $\strategy_i'$ and therefore should have preferred $\strategy_i'$. The problem is that the preference relation that this implies is not transitive, and is possibly cyclic. In Table 4 we can see that in the history $\play((\strategy_i',\jstrategy_{-i}),\pldomain)$ $i$ cannot give a strong excuse for not choosing $\strategy_i$ implying that $\strategy_i$ should be preferred to $\strategy_i'$. In $\play((\strategy_i'',\jstrategy_{-i}'),\pldomain)$ we can use a similar argument to imply that $\strategy_i'$ should be preferred to $\strategy_i''$. However, in $\play((\strategy_i,\jstrategy_{-i}''),\pldomain)$ $i$ we can argue that $\strategy_i''$ should be preferred to $\strategy_i$! Correspondingly, we will focus on weak excuses rather than strong excuses. In particular, we can refine the definition of passive responsibility by considering weak excuses.

\begin{definition}[Inexcusable Passive Responsibility]
    Given an MAS $\pldomain = (\game,s_0,k,\valueprof)$, an agent $i$ and a strategy $\jstrategy$ for $\agentset$, $i$ is attributed inexcusable passive responsiblility for $X \subseteq \valueprofplus$ in $\play(\jstrategy,\pldomain)$ if and only if there exists some $\strategy_i'$ such that $i$ is responsible for $X$ via $\strategy_i'$ and there is no weak excuse for $(i,\jstrategy,\pldomain,\strategy_i')$.
\end{definition}

Furthermore, we can show that this notion satisfies all of the important properties of responsibility that we have outlined thus far.

\begin{theorem}\label{thm:inexcusableproperties}
    Inexcusable passive responsibility satisfies consistency, completeness,the acceptance of weak excuses and the acceptance of strong excuses.
\end{theorem}

\begin{proof}
    Since $i$ bears inexcusable passive responsibility for $X$ only if $i$ is responsible for $X$ via some strategy $\strategy_i'$, we can use the same argument as for theorem \ref{thm:passive_resp_att} to show that inexcusable passive responsibility satisfies consistency and completeness. Inexcusable passive responsibility satisfies acceptance of weak excuses by definition and therefore also satisfies acceptance of strong excuses since all strong excuses are weak excuses.
\end{proof}

Again, by considering Table 3 we can show that passive responsibility does not satisfy acceptance of strong or weak excuses. Finally, we can show the notion of inexcusable passive responsibility is coherent with the single-value notion of liability that we introduced in Definition \ref{def:liability}:

\begin{theorem}
    Given an MAS $\pldomain = (\game,s_0,k,\valueprof)$, an agent $i$, a strategy $\jstrategy$ for $\agentset$ and a value $\omega \in \valueprof$, $i$ is liable for $\neg \omega$ in $\history = \play(\jstrategy,\pldomain)$ if and only if there is some $X \subseteq \valueprofplus$ such that $\neg \omega \in X$ and $i$ is attributed inexcusable passive responsibility for $X$ in $\history$.
\end{theorem}

\begin{proof}
    First we will show (liability $\Rightarrow$ inexcusable passive responsibility). Suppose that $i$ is liable for $\neg \omega$. Therefore $\history_1 = \play(\jstrategy,\pldomain) \models \neg \omega$ and there exists some strategy $\strategy_i'$ such that $\strategy_i <_\pldomain \strategy_i'$ and $\history_2 = \play((\strategy_i',\jstrategy_{-i}),\pldomain) \models \omega$. Therefore $\neg \omega \in X = \satset(\history_1,\valueprofplus) \setminus \satset(\history_2,\valueprofplus)$. Since  $\strategy_i <_\pldomain \strategy_i'$ there exists no weak excuse for $(i,\jstrategy,\pldomain,\strategy_i')$ and $\satset(\history_1,\valueprofplus) \preceq \satset(\history_2,\valueprofplus)$ (so therefore $X \preceq \emptyset$) so $i$ is attributed inexcusable passive responsibility for $X$ and we are done.

    Now we show (inexcusable passive responsibility $\Rightarrow$ liability). Suppose that $i$ is attributed inexcusable passive responsibility for $X$ and that $\neg \omega \in X$. Therefore we know that there is some strategy $\strategy_i'$ such that $X = \satset(\history_1,\valueprofplus) \setminus \satset(\history_2,\valueprofplus)$ where $\history_1 = \play(\jstrategy,\pldomain)$ and $\history_2 = \play((\strategy_i',\jstrategy_{-i}),\pldomain)$. Therefore $\history_1 \models \neg \omega$ and $\history_2 \models \omega$. Since we also know that there exists no weak excuse for $(i,\jstrategy,\pldomain,\strategy_i')$ we can also conclude that $\strategy_i \leq_\pldomain$ so $i$ is liable for $\neg \omega$ and we are done. 
\end{proof}

\begin{table}[]
    \centering
    \begin{tabular}{|c|c|}
    \hline
    Notion & Properties \\
    \hline
    Passive Responsibility & Consistency, Completeness\\
    \hline
    Inexcusable Passsive Responsibility & Consistency, Completeness,\\
    & Acceptance of Weak Excuses\\
    \hline
    \end{tabular}
    \caption{A table summarising the properties of Passive Responsibility (Theorem \ref{thm:passive_resp_att}) and Inexcusable Passive Responsibility (Theorem \ref{thm:inexcusableproperties}).}
    \label{tab:my_label}
\end{table}

\section{Anticipating Regret and Responsibility}\label{sec:regret}

Anticipation is the process of predicting the possible outcomes of a particular strategy. Anticipation is therefore performed before rather than after strategy execution, meaning that it can be used in the process of strategy selection.

\subsection{Regret Anticipation}

Roughly speaking, we want agents to select strategies that generally lead to good outcomes. One way to evaluate outcomes is simply to consider the set of values satisfied with that outcome in conjunction with the relation $\preceq$. However, it can instead be useful to consider how good an outcome is relative to the other outcomes that could have been achieved. In other words, if the strategies of the other agents guarantee that the outcome of any strategy is at best mediocre (relative to our value base $\valueprof$) then an agent should feel quite satisfied with a mediocre outcome. Alternatively, if the strategies of the other agents guarantee that the outcome is at \textit{worst} mediocre, then an agent should not be satisfied with a mediocre outcome. This leads to the notion of the regret, which is traditionally defined as the difference between the utility of a given outcome and the greatest utility that could have been achieved given the strategies of all other agents. This idea was first introduced in decision theory \cite{Savage51,Niehans1948} and was later formalised in  game theory \cite{Linhart89}. Regret is a very demanding notion, since the only way to achieve zero regret is to achieve the best result possible given the strategies of all other agents, and it does not give any ``extra credit'' for avoiding an even worse outcome.

Formalising a notion of regret in a setting with no numerical measurements of utility (such as ours) is more complex than normal, since we cannot simply subtract one value from another. However, we can simply consider the ``relative regret'' from $\history_1$ to $\history_2$ as $\satset(\history_1,\valueprofplus) \setminus \satset(\history_2,\valueprofplus)$. This should be understood as the regret that an agent would feel after $\history_1$ has occurred when they consider $\history_2$. For example, if $\valueprof = \{\omega_1,\omega_2\}$, $\history_1 \models \neg \omega_1 \land \neg \omega_2$ and $\history_2 \models \omega_1 \land \neg \omega_2$ then the regret from $\history_1$ to $\history_2$ is $\neg \omega_1$, since the agent regrets that $\omega_1$ was violated, but does not regret the violation of $\omega_2$ since this occurred in both histories.

We can also consider the regret associated with an individual strategy $\strategy_i$ for some agent $i$. Since there are many possible histories, and therefore many possible regret sets, that can be associated with $\strategy_i$, the standard approach in the literature is to select the worst value of regret that it could possibly experience. We call this ``anticipated regret'' because the anticipated regret represents a possible future result of executing the strategy, which can be calculated before actually executing the strategy and thus can be fruitfully used to evaluate strategies. The kind of strategies that agents should prefer are ``regret-minimising'' strategies.

\begin{definition}[Regret Minimisation]\index{Regret Minimisation}
    Given an MAS $\pldomain = (\game,s_0,k,\valueprof)$, an agent $i$ and a strategy $\strategy_i$, let $X$ be anticipated regret of $\strategy_i$ (the worst, according to $\preceq$, relative regret from any $\history_1 = \play((\strategy_i,\jstrategy_{-i}),\pldomain)$ to any $\history_2 = \play((\strategy_i',\jstrategy_{-i}),\pldomain)$). $\strategy_i$ is regret-minimising if and only if for any $\strategy_i'$, with anticipated regret $Y$, $Y \preceq X$.
\end{definition}

An advantage of symbolic regret over more standard, numerical forms of regret is that it can produce much more meaningful explanations of behaviour, particularly to a non-technical audience. A numerical justification for avoiding a particular strategy might be ``$\strategy'_i$ anticipates regret of $0.8$ compared to $0.4$ for $\strategy_i$'' whereas with symbolic regret we can say ``both $\strategy_i$ and $\strategy'_i$ lead to the avoidable violation of $\omega_1$ in the worst-case, but $\strategy_i$ at least has the compensation of satisfying $\omega_2$'' (i.e the anticipated regret of $\strategy_i$ is $\{\neg \omega_1, \omega_2\}$ whereas the anticipated regret of $\strategy'_i$ is $\{\neg \omega_1\}$).

\subsection{Responsibility Anticipation}\label{sec:responsibilityant}

Previous work on responsibility anticipation \cite{ECAIpaper} has only considered responsibility for single values, which simplifies the process of responsibility anticipation; if any possible history for $\strategy_i$ attributes responsibility for $\omega$, then $\strategy_i$ anticipates responsibility for $\omega$. This can be generalised to multiple values, but we must be careful about how we do it, as the following example shows.

\begin{table}[]
    \centering
\subfloat[Table 5]{
\begin{tabular}{|c|c|c|c|}
    \hline
    & $\strategy_{B}$ & $\strategy_{B}'$ & $\strategy_{B}''$\\
    \hline
    $\strategy_A$ & $\omega_1$, $\omega_2$ $\omega_3$, $\omega_4$ & $\omega_1$, $\omega_2$ $\omega_3$, $\omega_5$ &  $\omega_1$, $\omega_2$, $\omega_4$, $\omega_5$\\
    \hline
    $\strategy_A'$ & $\omega_1$, $\omega_2$ $\omega_3$ & $\omega_1$, $\omega_2$ $\omega_3$ & $\omega_1$, $\omega_2$ $\omega_3$ \\
    \hline
    $\strategy_A''$ & $\omega_5$ & $\omega_4$ & $\omega_3$ \\
    \hline
\end{tabular}}
\quad
\subfloat[Table 6]{
\begin{tabular}{|c|c|c|}
    \hline
    & $\strategy_{B}$ & $\strategy_{B}'$\\
    \hline
    $\strategy_A$ & $\emptyset$ & $\omega_1$\\
    \hline
    $\strategy_A'$ & $\emptyset$ & $\emptyset$\\
    \hline
\end{tabular}}
\end{table}
For example, we can see in Table 5 that the set of values that Anna is attributed responsibility for with $\strategy_A$ in at least one possible history is strictly larger than the equivalent set for $\strategy_A'$ ($\{\omega_3,\omega_4,\omega_5\}$ vs $\{\omega_4,\omega_5\}$), even though $\strategy_A$ strictly dominates $\strategy_A'$, due to always completing four of the tasks in the garden, instead of three. However, we can avoid this issue if we stick to responsibility sets from single histories.

Furthermore, to allow for strategy evaluation via considerations of responsibility it is useful to be able to pick a single set $X$ to represent the degree of responsibility that $i$ anticipates for choosing $\strategy_i$. Since we have no information regarding how the other agents may choose their strategies, the most obvious choices are to pick either the best or the worst of the sets for which $i$ is responsible\footnote{In the language of comparing sets of objects \cite{Barberà2004} we are in the situation of ``complete uncertainty'' meaning that most approaches revolve around the best and/or worst elements of the set.}. However, there are problems with using the best set.

In Table 6 we can see that in any history resulting from Ben choosing $\strategy_B$, Anna is responsible for $\emptyset$ regardless of which strategy she chooses, since Ben has made it impossible to complete any of the tasks. This means that the best responsibility sets for $\strategy_A$ and $\strategy_A'$ are both $\emptyset$, meaning that they would both be considered equally good in terms of anticipated responsibility, even though $\strategy_i$ weakly dominates $_\pldomain \strategy_i'$. Therefore I propose the following definition of responsibility anticipation:

\begin{definition}[Responsibility Anticipation]
    Given an MAS $\pldomain = (\game,s_0,k,\valueprof)$, an agent $i$, a strategy $\strategy_i$ and a type of responsibility $\Gamma$, $i$ anticipates $\Gamma$-responsibility for $X \subseteq \valueprofplus$ in $\play(\jstrategy,\pldomain)$ if and only if $X$ is the worst (according to $\preceq$) subset of $\valueprofplus$ such that there exists some joint strategy $\jstrategy_{-i}$ for $\agentset \setminus \{i\}$ such that $i$ is $\Gamma$-responsible for $X$ in $\play((\strategy,\jstrategy_{-i}),\pldomain)$.
\end{definition}

One feature of this definition, which will be useful in later proofs, is that responsibility anticipation is provably pessimistic, the best-case responsibility that an agent can anticipate is $\emptyset$ (neither positively nor negatively responible).

\begin{lemma}\label{lemma:min_responsibility_value}
    If $i$ anticipates passive or inexcusable passive responsibility for $X$ in $\pldomain$ with $\strategy_i$, then $X \preceq \emptyset$.
\end{lemma}

\begin{proof}
    If $\strategy_i$ anticipates responsibility for $X$ then $X$ must be the worst subset of $\valueprofplus$ such that there is some strategy $\jstrategy_{-i}$ such that $i$ is responsible for $X$ in $\history = \play((\strategy_i,\jstrategy_{-i}),\pldomain)$. In the history $\history$ $i$ will always be attributed both passive and inexcusable passive responsibility for $\emptyset$ via $\strategy_i$ since there can never be a weak excuse for $(i,\jstrategy,\pldomain,\strategy_i')$. Therefore the worst valid $X \subseteq \valueprofplus$ must be at most as good as $\emptyset$.
\end{proof}

A responsibility-conscious agent should prefer strategies where they minimise their degree of expected responsibility. This should ensure that they maximise their causal contribution towards a generally more positive outcome. 

\begin{definition}[Responsibility Minimisation]
     Given an MAS $\pldomain = (\game,s_0,$ $k,\valueprof)$, an agent $i$, a strategy $\strategy_i$ and a type of responsibility $\Gamma$, let $X \subseteq \valueprof$ be such that $\strategy_i$ anticipates $\Gamma$ responsibility for $X$. Then we say that $\strategy_i$ is $\Gamma$ responsibility-minimising for $i$ if and only if for every alternative strategy $\strategy_i'$ such that $i$ anticipates responsibility for $Y$, $Y \preceq X$.
\end{definition}

From this definition we can describe a responsibility-minimising strategy for the notions of passive and inexcusable passive responsibility. However, it also turns out that both of these notions can be described as some combination of non-dominated and regret-minimising strategies.

\begin{theorem}\label{thm:passiveequivalence}
   Given an MAS $\pldomain = (\game,s_0,k,\valueprof)$, an agent $i$ and a a strategy $\strategy_i$, $\strategy_i$ is passive responsibility-minimising if and only if $\strategy_i$ is regret-minimising.
\end{theorem}

\begin{proof}
    To prove this, it is sufficient to prove that given an MAS $\pldomain = (\game,s_0,k,\valueprof)$, an agent $i$ and a strategy $\strategy_i$, the anticipated passive responsibility for $i$ with $\strategy_i$ is exactly the anticipated regret for $i$ with $\strategy_i$.

    Let $X$ be the anticipated passive responsibility for $i$ with $\strategy_i$ in $\pldomain$. Therefore $X$ is the worst (according to $\preceq$) subset of $\valueprof$ such that there exists some strategy $\jstrategy_{-i}$ such that $i$ is passively responsible for $X$ in $\play((\strategy_i,\jstrategy_{-i}),\pldomain)$. This means that $X$ is the worst subset of $\valueprof$ such that there exists strategy $\jstrategy_{-i}$ for $\agentset \setminus \{i\}$ and  a strategy $\strategy_i'$ for $i$ such that $X = \satset(\play((\strategy_i,\jstrategy_{-i}),\pldomain),\valueprofplus) \setminus \satset(\play((\strategy_i',\jstrategy_{-i}),\pldomain),\valueprofplus)$. However, this is enough to show that $X$ is also the anticipated regret of $i$ with $\strategy_i$ in $\pldomain$.
\end{proof}

\begin{theorem}\label{thm:inexcusableequivalence}
    Given an MAS $\pldomain = (\game,s_0,k,\valueprof)$, an agent $i$ and a a strategy $\strategy_i$, $\strategy_i$ is inexcusable passive responsibility-minimising if and only if $\strategy_i$ is non-dominated.
\end{theorem}

\begin{proof}
    First we will show that (non-dominated $\Rightarrow$ inexcusable passive responsibility-minimising). Suppose that $\strategy_i$ is non-dominated for $i$ in $\pldomain$. Let $X \subseteq \valueprofplus$ be such that $X \prec \emptyset$ and there is some strategy $\jstrategy_{-i}$ for $\agentset \setminus \{i\}$ and strategy $\strategy_{i}'$ for $i$ such that $i$ is responsible for $X$ via $\strategy_i'$ in $\play((\strategy_i,\jstrategy_{-i}),\pldomain)$. Since $X \prec \emptyset$ we can conclude that $\satset(\play((\strategy_i,\jstrategy_{-i}),\pldomain),\valueprofplus) \prec \satset(\play((\strategy_i',\jstrategy_{-i}),\pldomain),\valueprofplus)$. Since $\strategy_i$ is non-dominated there must also exist some $\jstrategy_{-i}'$ such that\\ $\satset(\play((\strategy_i',\jstrategy_{-i}'),\pldomain),\valueprofplus) \prec \satset(\play((\strategy_i,\jstrategy_{-i}'),\pldomain),\valueprofplus)$. Therefore $\jstrategy_{-i}'$ is a weak excuse for $(i,(\strategy_{i},\jstrategy_{-i}),\pldomain,\strategy_i')$. This means that $i$ has a weak excuse for any $X \prec \emptyset$ that $i$ would otherwise be responsible for. Therefore by Lemma \ref{lemma:min_responsibility_value} we have that $i$ anticipates responsibility for $\emptyset$ with $\strategy_i$ and that $\strategy_i$ is inexcusable passive responsibility-minimising.

    Now we show that (not non-dominated $\Rightarrow$ not inexcusable passive responsibility-minimising). Suppose that $\strategy_i$ is asymmetrically weakly dominated by some strategy $\strategy_i'$. Therefore there is some strategy $\jstrategy_{-i}$ for $\agentset \setminus \{i\}$ such that\\  $\satset(\play((\strategy_i,\jstrategy_{-i}),\pldomain),\valueprofplus) \prec \satset(\play((\strategy_i',\jstrategy_{-i}),\pldomain),\valueprofplus)$ and there is no strategy $\jstrategy_{-i}'$ for $\agentset \setminus \{i\}$ such that  $\satset(\play((\strategy_i',\jstrategy_{-i}'),\pldomain),\valueprofplus) \prec$\\ $\satset(\play((\strategy_i,\jstrategy_{-i}'),\pldomain),\valueprofplus)$. Therefore we know that $X =$\\ $\satset(\play((\strategy_i,\jstrategy_{-i}),\pldomain),\valueprofplus) \setminus \satset(\play((\strategy_i',\jstrategy_{-i}),\pldomain),\valueprofplus) \prec \emptyset$ and that there cannot exist an excuse for $(i,(\strategy_{i},\jstrategy_{-i}),\pldomain,\strategy_i')$. Therefore $\strategy_i$ anticipates inexcusable passive responsibility for $X$ where $X \prec \emptyset$.

    Furthermore we know that there must exist a strategy $\strategy_i''$ for $i$ in $\pldomain$ that is non-dominated and by the above argument we see that $\strategy_i''$ anticipates responsibility for $\emptyset$. Therefore $\strategy_i$ is not passive responsibility-minimising.
\end{proof}

Since regret-minimisation and being non-dominated are both desirable properties, we argue that agents attempting to minimise their anticipated responsibility should attempt to minimise both passive and inexcusable passive responsibility. We can also show that minimising one notion never comes at the cost of failing to minimise the other. The requirement that agents should minimise both passive and inexcusable passive responsibility is consistent with the idea that having an excuse reduces an agent's degree of responsibility, but may not remove it entirely. 

\begin{theorem}\label{thm:alwaysoverlap}
    Given an MAS $\pldomain = (\game,s_0,k,\valueprof)$ and an agent $i$, there exists some strategy $\strategy_i$ that is both regret-minimising and non-dominated for $i$.
\end{theorem}

\begin{proof}
Since each strategy has an ``anticipated regret'' and there are finitely many distinct strategies, there must exist a regret-minimising strategy $\strategy_i$. It is enough to show that any strategy that asymmetrically weakly dominates $\strategy_i$ must also be regret-minimising.

Let $X$ be the anticipated regret for $\strategy_i$. Suppose that $\strategy_i' <_\pldomain \strategy_i$ and let $X'$ be the anticipated regret for $\strategy_i'$. Suppose for contradiction that $\strategy_i'$ is not regret-minimising, meaning that $X' \prec X$. We also know that there exists some strategy $\jstrategy_{-i}$ for $\agentset \setminus \{i\}$ such that $X' = \satset(\play((\strategy_i',\jstrategy_{-i}),\pldomain),\valueprofplus) \setminus \satset(\history_1,\valueprofplus)$ where $\history_1$ is the best (according to $\preceq$) history compatible with $\jstrategy_{-i}$ and $\pldomain$. Let $Y = \satset(\play((\strategy_i,\jstrategy_{-i}),\pldomain),\valueprofplus) \setminus \satset(\history_1,\valueprofplus)$. Since $X$ is the anticipated regret of $\strategy_i$ we know that $X \preceq Y$ and therefore $X' \prec Y$.

Since $X' = \satset(\play((\strategy_i',\jstrategy_{-i}),\pldomain),\valueprofplus) \setminus \satset(\history_1,\valueprofplus)$, and $Y = $\\ $\satset(\play((\strategy_i,\jstrategy_{-i}),\pldomain),\valueprofplus) \setminus \satset(\history_1,\valueprofplus)$ and $X' \prec Y$ we can conclude that $\satset(\play((\strategy_i',\jstrategy_{-i}),\pldomain),\valueprofplus) \prec \satset(\play((\strategy_i,\jstrategy_{-i}),\pldomain),\valueprofplus)$\footnote{The proof of the validity of this deduction can be found in the appendix.}. However, this is a contradiction as we supposed that $\strategy_i <_\pldomain \strategy_i'$. Therefore $\strategy_i$ must be regret-minimising and we are done.
\end{proof}

\begin{figure}[h!]
\begin{center}
\includegraphics[scale = 0.75]{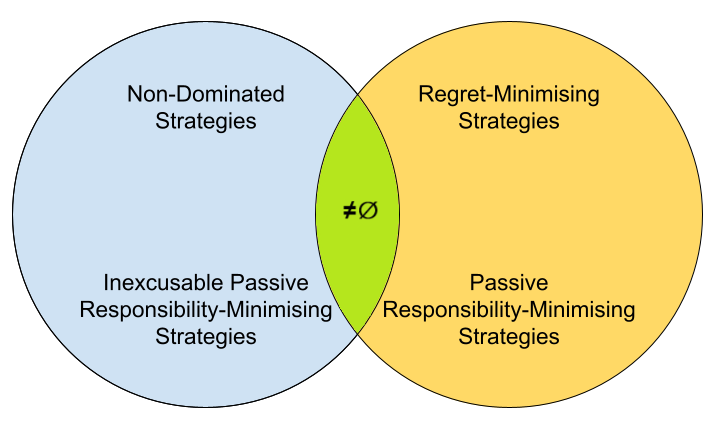}
\caption{A visual summary of the main results of this chapter. The two circles indicate that regret-minimising strategies are exactly passive resbonsibility-minimising strategies (Theorem \ref{thm:passiveequivalence}) and that non-dominated strategies are exactly inexcusable passive responsibility-minimising strategies (Theorem \ref{thm:inexcusableequivalence}). The intersection of these two circles is always non-empty (Theorem \ref{thm:alwaysoverlap}).}\label{figure3}
\end{center}
\end{figure}

\section{Conclusion and Future Work}\label{sec:futurework}

In this paper we have introduced a formal setting for the study of multi-value responsibility. We have considered various desirable properties for notions of multi-value responsibility, namely consistency, completeness and the acceptance of weak and strong excuses, and then defined two notions of responsibility (passive and inexcusable passive responsibility) that satisfy some or all of these properties. We introduced a symbolic notion of regret as well as the notion of non-dominance in this setting. Finally we applied the concept of responsibility anticipation to our notions of responsibility and showed that the anticipation of each notion can be alternatively formalised in terms of regret minimisation and non-dominance.

We argue that value-based agents should always seek to minimise both passive and inexcusable passive responsibility since this guarantees both regret minimising and non-dominated strategies. On the other hand, it is not clear that strategies that minimise worst-case regret are always to be preferred, even if they are also best-effort, since there may exist an alternative best-effort strategy that does not minimise regret in the worst case but does a much better job of minimising regret ``overall''. However, given that our model contains no information that $i$ can use to evaluate the relative likelihood of different strategies for $\agentset \setminus \{i\}$, it is difficult to formalise what ``overall'' means. Nonetheless, for some applications we may prefer to require only that agents minimise inexcusable passive responsibility instead of passive responsibility, since selecting non-dominated strategies will always be preferable (since by definition, an agent cannot give an excuse for not choosing a non-dominated strategy).

Overall, we hope that this paper has demonstrated the potential application of responsibility-conscious agents to decision-making in strategic multi-value scenarios. By demonstrating the connections between considerations of responsibility and pre-existing notions such as non-dominated and regret-minimising strategies, we aim to show both how responsibility considerations can lead to better decision-making, but also how good decision-making works to generally decrease the degree of responsibility of an agent.

The application of this work is most natural in settings where agents are working towards different goals but are not directly competing. Our running example of a cleaning robot in an airport is such a setting, as the other agents (the passengers) do not share the goals of the robot, but also have no desire to prevent the robot from achieving its goals. Responsibility-based evaluation is also particularly useful when the satisfaction of some or all of an agents goals or values cannot be guaranteed (due to possible interference from other agents), since responsibility minimisation \textit{can} always be guaranteed.

In future work, we would like to consider more closely the connection between responsibility for risk and responsibility anticipation, since both acknowledge that an agent may need to consider/may be attributed responsibility for some outcome even in scenarios where the outcome in question does not occur, but where the actions of the agent mean that it \textit{might} have occurred. This would be particularly relevant if we introduce explicitly probabilistic reasoning into our model, as this is a typical feature of the formal study of risk.

Since (to the best of our knowledge) the field of research into multi-value responsibility is at present extremely limited, there are a number of potentially fruitful directions for future research. Our model makes a number of simplifying assumptions, in particular assuming that all agents have perfect knowledge of the domain $\pldomain$ but have no knowledge about the relative likelihood of the possible strategies of the other agents. Both of these assumptions could be dropped, such as by giving the agents only partial information about the start state/action theory, or by letting them reason in some way about the values and rationality of other agents. It would also be useful to have a more precise computational grounding as both a first step towards a real-world implementation of these ideas and a useful insight into the computational complexity of this model. This would require some changes to the model, such as a more complex represantation of strategies to make iterating over all possible strategies computationally feasible. It is also worth expanding our definitions of responsibility to consider responsibility for multiple agents (where this paper is focused on a single-agent perspective), such as when some coalition of agents jointly acts to guarantee outcomes that none of them could have guaranteed individually (as others have done for single-value responsibility). Finally, we believe that there is much more work to be done both in studying the formal properties of the notions of multi-value responsibility that we present here, and further exploring the undoubtedly rich landscape of responsibility notions in multi-value settings.

\begin{credits}

\subsubsection{\discintname}
The authors have no competing interests.
\end{credits}
\bibliographystyle{splncs04}
\bibliography{biblio}

\section*{Appendix}

\begin{corollary}
    Given Histories $\history_1, \history_2, \history_3$, and a value base $\valueprof$, let $X = \satset(\history_1,\valueprofplus)$, $Y = \satset(\history_2,\valueprofplus), Z = \satset(\history_3,\valueprofplus)$. Then $X \preceq Y$ if and only if $X\setminus Z \preceq Y \setminus Z$. 
\end{corollary}

\begin{proof}
    Let $k$ be the number of priority levels for $\valueprof$ (so $\valueprof = (\valueset_1, \ldots, \valueset_k)$. For all $1 \leq k' \leq k$, let $Score(A,k') = |A \cap \valueset_{k'}| - |A \cap \valueset_{k'}^\neg|$. Then we can see that comparing any two sets $A$ and $B$ by $\preceq$ depends solely on the values of $Score(A,k')$ and $Score(B,k')$ for $1 \leq k' \leq k$.

    Therefore it is enough to prove that for all $1 \leq k' \leq k$, $Score(X,k') \leq Score(Y,k')$ if and only if $Score(X \setminus Z,k') \leq Score(Y \setminus Z,k')$.

    Let $k'$ be such that $1 \leq k' \leq k$, let let $t = |\valueset_k'|$. Let $x = |X \cap \valueset_{k'}|$ and $z = |Z \cap \valueset_{k'}|$. Note that $|X \cap \valueset_{k'}^{\neg}| = t - x$, since all values in $\valueset_{k'}$ are either satisfied or violated in $X$. Therefore $Score(X,k') = 2x-t$ and by a similar argument $Score(Z,k') = 2z-t$.

    Let $c = |\valueset_k' \cap X \cap Z|$. Therefore $|\valueset_{k'} \cap (X \setminus Z)| = x - c$ and $|\valueset_{k'} \cap (Z \setminus X)| = z - c$. Note that $|\valueset_{k'}^\neg \cap (X \setminus Z)| = |\valueset_{k'} \cap (Z \setminus X)|$ since every value satisfied in $Z$ but not in $X$ is a value that is violated in $X$ but not in $Z$. Therefore 
    \begin{align*}
    &Score(X \setminus Z,k')\\
    =&|\valueset_{k'} \cap (X \setminus Z)| - |\valueset_{k'}^\neg \cap (X \setminus Z)|\\
    =&(x-c) - (z-c)\\
    =&x-z\\
    =&\frac{2x-2z}{2}\\
    =&\frac{(2x-t)-(2z-t)}{2}\\
    =&\frac{Score(X,k') - Score(Z,k')}{2}
    \end{align*}

By a similar argument $Score(Y \setminus Z,k') = \frac{Score(Y,k') - Score(Z,k')}{2}$. It follows immediately that $Score(X,k') \leq Score(Y,k')$ if and only if $Score(X \setminus Z,k') \leq Score(Y \setminus Z,k')$ and we are done.
\end{proof}
\end{document}